\documentclass{article}
\pdfoutput=1
\usepackage{arxiv}
\usepackage{bm}
\usepackage[utf8]{inputenc} 
\usepackage[T1]{fontenc}    
\usepackage{hyperref}       
\usepackage{url}            
\usepackage{booktabs}       
\usepackage{amsfonts}       
\usepackage{nicefrac}       
\usepackage{microtype}      
\usepackage{lipsum}		
\usepackage{graphicx}
\usepackage{amsmath}
\usepackage{amsthm}
\usepackage{float}
\usepackage{multirow}

\newtheorem{cor}{Corollary}

\newtheorem{lem}{Lemma}
\newtheorem{prop}{Proposition}

\newtheorem{assump}{Assumption}

\begin{document}
\title{A regularized deep matrix factorized model of matrix completion for image restoration}


\author{
  Zhemin Li \thanks{lizhemin@nudt.edu.cn}\\
  College of Arts and Sciences, National University of Defense Technology\\
  \AND
  Zhi-Qin John Xu\\
  School of Mathematical Sciences,\\
  MOE-LSC and Institute of Natural Sciences,\\
  Shanghai Jiao Tong University.\\
   \AND
   Tao Luo \\
  School of Mathematical Sciences,\\
MOE-LSC and Institute of Natural Sciences,\\
Shanghai Jiao Tong University.\\
   \AND
   Hongxia Wang\thanks{Corresponding author: wanghongxia@nudt.edu.cn}\\
   College of Arts and Sciences, National University of Defense Technology\\
}



\maketitle

\begin{abstract}
It has been an important approach of using matrix completion to perform image restoration. Most previous works on matrix completion focus on the low-rank property by imposing explicit constraints on the recovered matrix, such as the constraint of the nuclear norm or limiting the dimension of the matrix factorization component. Recently, theoretical works suggest that deep linear neural network has an implicit bias towards low rank on matrix completion. However, low rank is not adequate to reflect the intrinsic characteristics of a natural image. Thus, algorithms with only the constraint of low rank are insufficient to perform image restoration well. In this work, we propose a Regularized Deep Matrix Factorized (RDMF) model for image restoration, which utilizes the implicit bias of the low rank of deep neural networks and the explicit bias of total variation. We demonstrate the effectiveness of our RDMF model with extensive experiments, in which our method surpasses the state of art models in common examples, especially for the restoration from very few observations. Our work sheds light on a more general framework for solving other inverse problems by combining the implicit bias of deep learning with explicit regularization.

\end{abstract}


\section{Introduction}
Image restoration from a partially observed pixel is an important problem. An image is often represented by a matrix. Therefore, image restoration is also a matrix completion, that is, recovering unknown or missing entries in a matrix.  In addition to image restoration \cite{Li2015NonLocalII,Wang2014RankAA,Wu2015ImageCW}, matrix completion problem has a wide spectrum of application venues such 
as phase recovery \cite{Cands2015PhaseRV} and recommendation systems \cite{Bojnordi2012ANC,Kang2016TopNRS}. The prior of low rank has been well studied, and a series of algorithms are developed to well perform matrix completion for low-rank matrices. These algorithms can be classified into two types. The first type imposes explicit constraints on the recovered matrix, such as the constraint of the nuclear norm or limiting the dimension of the matrix factorization component. Such explicit constraints are proved to be successful in many problems \cite{Cai2010ASV,Wen2012SolvingAL,Geng2014ANA,Fan2018MatrixCB}. The second type utilizes implicit bias of algorithms. Such implicit bias is recently studied in a deep linear neural network, in theory, i.e., theoretical works suggest that deep linear neural network has an implicit bias towards low rank on matrix completion \cite{gunasekar2017implicit,Arora2019ImplicitRI}. However, it is yet to examine its effectiveness in the real application. Low-rank matrix completion has been used in image restoration, while low rank is not adequate to reflect the intrinsic characteristics of a natural image. For example, a restored image with a lower rank can significantly differ with the corresponding intact image. Therefore, it is important to design an algorithm that takes advantage of the study of low-rank matrix completion and preserves the intrinsic properties of a natural image.

In this paper, we design a Regularized Deep Matrix Factorized (RDMF) model to perform the specific matrix completion, that is, image restoration. The idea of this RDMF model is combining a deep neural network with a total variation (TV) regularization. To demonstrate the effectiveness of our model, we compare our model with three representative methods for the restoration of two classic images, i.e., Cameraman and textual image. The first method is a popular linear method, i.e., LaMFit \cite{Wen2012SolvingAL}. The second method is also deep matrix factorization (DMF) without TV constraint \cite{Fan2018MatrixCB}. The third method is a patch-based nonlinear matrix completion (PNMC) algorithm for the image, which is recently proposed and compared with $12$ other methods in  \cite{Yang2020ANP}. We demonstrate the effectiveness of our RDMF model with extensive experiments, in which our method surpasses these three models in common examples. The restorations of LaMFit and DMF indeed satisfy the low-rank requirement. However, their results have many vertical and horizontal lines, which significantly different from natural images. The result of PNMC shows much more noise. The advantage of our model is significant for the restoration from very few observations. In our experiments, we found that to achieve a good result, the depth of the deep neural network should be larger than two, and the width should be at least around the image size. These results indicate the implicit bias in deep networks with TV better reflects the characteristics of a natural image. In addition, we found that in our model, both deep linear neural networks and nonlinear neural networks have similar good performance, namely, in such image restoration, the benefit from the non-linearity may be incremental. 

The remainder of this paper is organized as follows:
in Section 2, we give the RDMF model firstly, and then we introduce
gradient descent algorithm.
In Section 3, we design several experiments to illustrate the critical components of RDMF.
In Section 4, we discuss the underlying mechanism of the RDMF.
In Section 5, we state conclusions and future research plans.

\section{Model formulation}\label{sec..model}
In this section, we mainly discuss how to build gradient descent based matrix completion models.
\subsection{Regularized deep matrix factorized model}

Before building the model, here we give some more notations, 
which are slightly different from the suggested notations for machine learning \cite{beijing2020Suggested}.
An $L$-layer neural network is denoted by
\begin{equation}
    \label{general dmf}
    f_{\bm{\theta}}(\bm{X})=\sigma\circ\left(\bm{W}^{[L-1]} \sigma \circ\left(\bm{W}^{[L-2]} \sigma \circ\left(\cdots\left(\bm{W}^{[1]} \sigma \circ\left(\bm{W}^{[0]} \bm{X}+\bm{b}^{[0]}\right)+\bm{b}^{[1]}\right) \cdots\right)+\bm{b}^{[L-2]}\right)+\bm{b}^{[L-1]}\right),
\end{equation}
where $\bm{X} \in\mathbb{R}^{m_{0}\times m_{0}} $,  $\bm{W}^{[l]}\in\mathbb{R}^{m_{l+1}\times m_{l}},\bm{b}^{[l]}\in\mathbb{R}^{m_{l+1}\times m_{0}},m_0=d,m_L=d_{\rm o},\sigma$ is a scalar function
and "$\circ$" means entry-wise operation.
We denote the vector of all parameters by
$$
    \bm{\theta}=\mathrm{vec}\left(\bm{W}^{[0]}, \bm{W}^{[1]}, \ldots, \bm{W}^{[L-1]}, \bm{b}^{[0]}, \bm{b}^{[1]}, \ldots, \bm{b}^{[L-1]}\right)
$$
and an entry of $\bm{W}^{[l]}$ by $W_{ij}^{[l]}$.


Assume that the original data matrix without missing entry is $\bm{X}\in\mathbb{R}^{d_{\rm o}\times d}$.
The $(i,j)$-th entry of $\bm{X}$ is $X_{ij},(1\leq i\leq d_{\rm o},1\leq j\leq d)$.
The mask matrix $\bm{\Omega}$ is defined as
$$
    \Omega_{i j}=\left\{\begin{array}{ll}
    1, & \text { if } X_{i j} \text { is known }, \\
    0, & \text { if } X_{i j} \text { is missed }.
    \end{array}\right. 
$$
And the fidelity term for completing $\bm{X}$ by $f_{\bm{\theta}}$ is
\begin{equation}
    \label{eq:implicit}
    R_{\bm{\Omega}}\left(\bm{X},\bm{\theta}\right) =
    \left\|\bm{\Omega}\odot\left(\bm{X}-f_{\bm{\theta}}(\bm{I}_d)\right)\right\|_\mathrm{F}^2,
\end{equation}
where  $\bm{I}_d$ is a $d$ dimensions identity matrix, $\odot$ stands for
the Hadamard (entry-wise) product.

A regularized model is to solve the following problem:
\begin{equation}
	\label{eq:rdmf}
	R(\bm{X},\bm{\theta})=R_{\bm{\Omega}}(\bm{X},\bm{\theta})+\lambda R_\mathrm{reg}(\bm{X},\bm{\theta}),
\end{equation}
where $R_\mathrm{reg}$ is the regularization term.
Given some observation $\bm{X}|_{\bm{\Omega}}$, the matrix completion problem can be solved by
\begin{align*}
    \hat{\bm{\theta}}=\arg\min_{\bm{\theta}}R(\bm{X},\bm{\theta})
    &=R_{\bm{\Omega}}(\bm{X},\bm{\theta})+\lambda R_\mathrm{reg}(\bm{X},\bm{\theta}),\\
    \hat{\bm{X}}
    &=f_{\hat{\bm{\theta}}}(\bm{I}_d).
\end{align*}

When we set $\sigma(\cdot)=\cdot,\bm{b}^{[l]}=0,R_\mathrm{reg}=0$,
the formulation of the deep network is same as  
\begin{equation}
\label{eq:linear dmf}
f_{\bm{\theta}}(\bm{I}_d)=\bm{W}^{[L-1]}\bm{W}^{[L-2]}\ldots\bm{W}^{[1]}\bm{W}^{[0]}
=\bm{W},
\end{equation}
where $f_{\bm{\theta}}(\bm{I}_d)$ is the recovered matrix.

Note that DMF in \cite{Fan2018MatrixCB} share same columns in $\bm{b}^{[l]}$
and set $R_\mathrm{reg}=\sum_{l=0}^{L-1}\left\|\bm{W}^{[l]}\right\|_\mathrm{F}^2$.
DMF is most similar to our model,
but their model requires $m_1<m_2<\ldots<m_L$ to constraint the rank of $f_{\bm{\theta}}(\bm{I}_d)$ explicitly.
Our model do not constraint on the dimension of $m_l$,
as Arora et al. \cite{Arora2019ImplicitRI} illustrate that
Eq. (\ref{eq:linear dmf}) implicitly convergence to low rank matrix.
A more important difference between DMF and RDMF is that we set 
regularization term as TV which will be introduced next section.

The TV regularization is often used in image processing, which tends to a smooth output \cite{Rudin1992NonlinearTV}. The TV norm can either be the anisotropic TV norm
$$
    \left\|f\right\|_\mathrm{TV_1}=\sum_{i,j}\left(\left|D_x f(i,j)\right|
    +\left|D_y f(i,j)\right|\right),
$$
or the isotropic TV norm
$$
    \left\|f\right\|_\mathrm{TV_2}=\sum_{i,j}\sqrt{(D_x f(i,j))^2+(D_y f(i,j))^2},
$$
where the finite differences $D_x f(i,j)=f(i+1,j)-f(i,j)$ and $D_y f(i,j)=f(i,j+1)-f(i,j)$.
Note that the anisotropic (isotropic) TV norm $\|f\|_\mathrm{TV}$ is equivalent to discretization of
the 1-norm (2-norm) of $\nabla f$.
They are TV-L1 and TV-L2 respectively in this paper.

We propose our RDMF model as follow:
\begin{equation}
\label{eq:main model}
	\hat{\bm{\theta}}=\arg\min_{\bm{\theta}} R(\bm{X},\bm{\theta})=R_{\bm{\Omega}}+\lambda R_\mathrm{TV}=\left\|\bm{\Omega}\odot\left(\bm{X}-f_{\bm{\theta}}(\bm{I}_d)\right)\right\|_\mathrm{F}^2+\lambda \left\|f_{\bm{\theta}}(\bm{I}_d)\right\|_\mathrm{TV},
	\hat{\bm{X}}=f_{\hat{\bm{\theta}}}(\bm{I}_d).
\end{equation}

Compared with TV-L2, the TV-L1 norm leads to a sharper edge. 
In the noisy image recover problem, 1-norm
TV norm is often chosen.
However, in the matrix completion problem, we find out that TV-L2 outperformed TV-L1 in most tasks.
Without loss of generality, we consider TV regularizer both including TV-L1 and TV-L2.

\subsection{Gradient descent}  
We introduce gradient descent method firstly to build an 
intuitive understanding of the optimization progress.
At time step $t$, the parameters update as follow:
$$
    \bm{\theta}_{t+1} =\bm{\theta}_t -
    \eta\nabla_{\bm{\theta}}R,
$$
where $\theta_t$ is the value of parameters at the $t$th iteration,
$\eta$ is the learning rate.
The vanilla gradient descent algorithm
has many drawbacks such as convergence speed is slow and 
easy to fall into the local optimal.
As a result, many variants of gradient descent are proposed to improve 
the gradient descent algorithm.
Adam \cite{Kingma2015AdamAM} is one of them and
performs well in practical training tasks.
Therefore, we choose the Adam algorithm to solve our proposed model \ref{eq:main model}.
\section{Experiment}
\label{sec:exp}
In this section, we first show the effectiveness of RDMF over peer methods. 
\subsection{Settings}
\textbf{Task} We compare RDMF with LaMFit \cite{Wen2012SolvingAL},
DMF \cite{Fan2018MatrixCB} and PNMC \cite{Yang2020ANP} both on natural and textual images.
Similar to the tests in \cite{Yang2020ANP}, all of these images are $240\times 240$ gray-scale images scaled from standard $512\times 512$ images.
We generate a mask matrix $\bm{\Omega}$ with different missing percentage randomly.

\textbf{Peer methods}
\begin{enumerate}
\item LaMFit \cite{Wen2012SolvingAL}: Based on the matrix decomposition model, LaMFit
constructs a super-relaxation algorithm, which only
needs to solve a least-squares problem in each iteration and avoids the 
the computation consuming SVD in the kernel-norm based method.
\item DMF \cite{Fan2018MatrixCB}: This algorithm utilizes deep neural 
network to recover missing entries without considering regularization.
\item PNMC \cite{Yang2020ANP}: Considering the spatial locality of the datasets, convolutional neural network (CNN) is designed
to obtain the matrix with missing entries.
\item RDMF: The method proposed in this paper. We default choose TV as regularization term and $\sigma(\bm{x})=\bm{x}$.
\end{enumerate}

\textbf{Criteria}
In order to evaluate the accuracy of the matrix completion, we adopt 
Normalized Mean Absolute Error (NMAE) \cite{Wen2012SolvingAL} as a criteria.
NMAE is defined as
$$
    \mathrm{NMAE}=\frac{1}{\left(X_{\max }-X_{\min }\right)|\bm{\Omega}^c|}
    \sum_{(i, j) \in \bm{\Omega}^c} \left | \hat{X}_{i j}-X_{i j}\right |,
$$
where $\hat{\bm{X}},\bm{X},\bm{\Omega}^c,|\bm{\Omega}^c|,X_{\max}$ and $X_{\min}$ denote the
prediction matrix, the original matrix, the mask matrix of test datasets, 
the cardinality of $\bm{\Omega}^c$, the maximum entry of matrix $\bm{X}$ and the 
minimum entry of matrix $\bm{X}$, respectively. 
Smaller values of NMAE indicate better predictive accuracy than the larger one.

\textbf{Model parameters settings}
The hyper-parameters for the proposed method
are set as follows:
$\eta = 0.001$, parameters initialization Gaussian distribution of $\mathcal{N}(0,10^{-3})$.
Without special instruction, we default set $L=3, m_i=240,\lambda=\frac{1}{240}$.
The algorithm stop when iteration step is larger than 10000 or 
the training loss $\left|R(\bm{X},\bm{\theta}_t)-R(\bm{X},\bm{\theta}_{t+1})\right|<10^{-3}$,
where $R(\bm{X},\bm{\theta}_t)$ is the value of $R$ at $t$-th iteration.
Besides, the proposed algorithm is implemented in the Pytorch framework.
All the simulations are conducted on the same workstation with an Intel(R)
Xeon(R) Silver 4110 CPU @ 2.10GHz, Nvidia GeForce GRX 2080Ti, running with Linux and using Python tool.

\subsection{Effectiveness of RDMF}
Firstly we randomly drop pixels in Cameraman and textual image,
and then recover them by LaMFit \cite{Wen2012SolvingAL}, DMF \cite{Fan2018MatrixCB}, PNMC \cite{Yang2020ANP}. 
Fig. \ref{fig:recovered} shows the recovered result of different methods.
We can observe that both DMF and PNMC perform better than LaMFit.
In \cite{Fan2018MatrixCB,Yang2020ANP}, Fan et al. and Yang et al. 
claim that the improvement of recovering is derived from the nonlinearity of the model. However, 
in this experiment, we choose $\sigma(\bm{x})=\bm{x}$ as the activation function of RDMF model, which is degenerate to a linear model.
The restored image is much better than all of these models. This indicates that nonlinearity is not always important in image restoration. 
As shown in Fig. \ref{fig:recovered} (e), RDMF recover more details than PNMC.
Specially, we can identify the text on the recovered result of the textual image.
These results illustrate that a linear model with the TV regularization term can well perform the image restoration, even better for a non-linear model.

To explore the effectiveness of the RDMF model, we calculate the NMAE of three peer methods
and RDMF with different activation functions.
Both of the RDMF with and without TV regularization terms are calculated.
Fig. \ref{fig:namecameraman} shows the NMAE of restored Cameraman at different missing percentage and models.
Tab. \ref{tab:nmae of text} shows the NMAE of the restored textual image at different missing percentage and models.
We observe that almost all the RDMF model with TV regularization have a lower NMAE than PNMC.
Even the linear model of RDMF with TV regularization outperforms other methods.
The linear model of RDMF with TV regularization is very similar to the nonlinear one.
Note that LaMFit performs better than RDMF of Tanh and ReLU activation function without TV regularization term. Therefore, TV regularization is very important in this image restoration.

\begin{figure}
	\centering
	\includegraphics[width=0.8\linewidth]{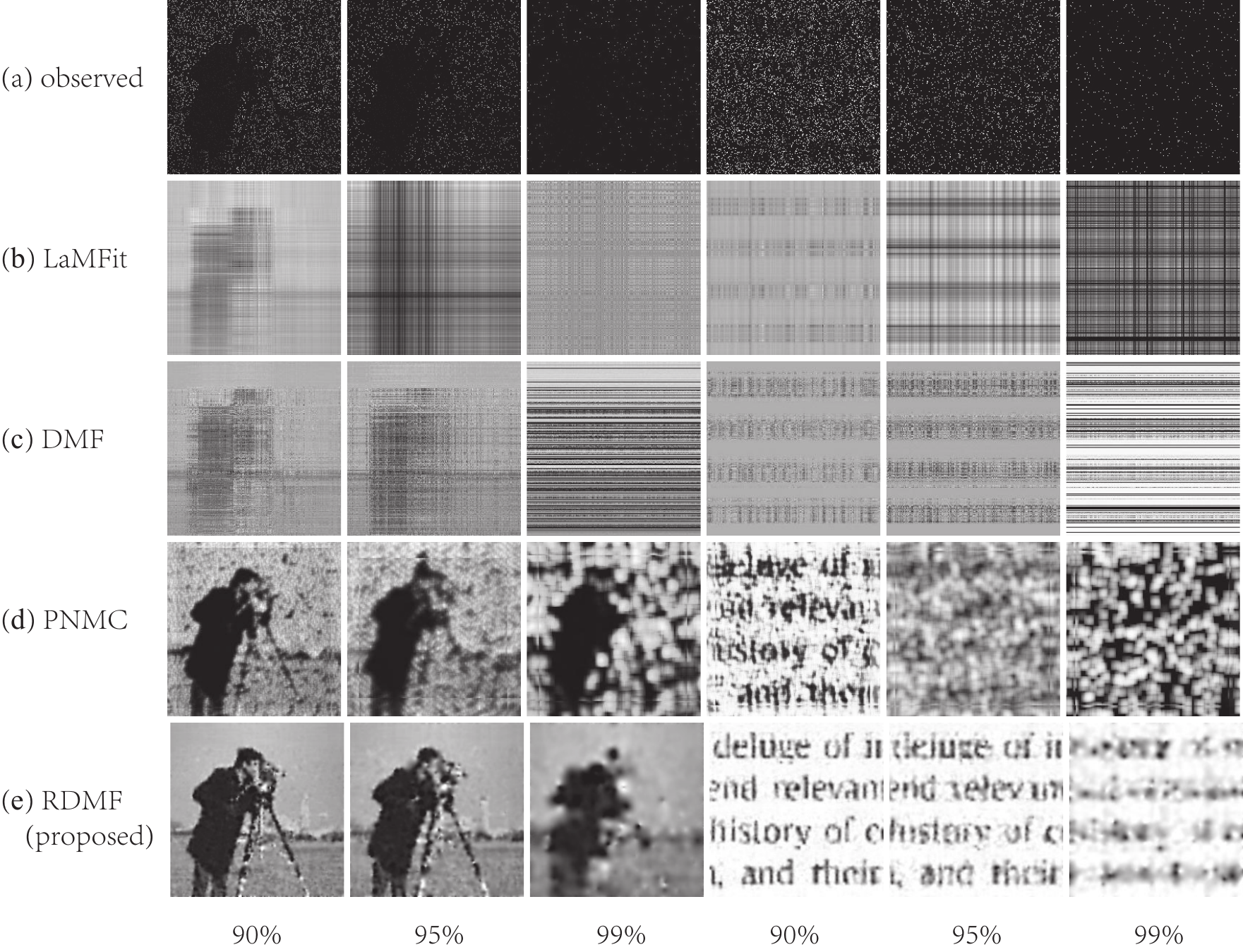}
	\caption{image restoration results for the Cameraman and the textual image with $90\%,95\%,99\%$ of pixels missing from left to right, respectively. The input images with randomly missed pixels are shown in (a). Followed by results respectively are LaMFit (b) \cite{Wen2012SolvingAL}, DMF (c) \cite{Fan2018MatrixCB}, PNMC (d) \cite{Yang2020ANP}, proposed RDMF (e).
		RDMF model choose the linear activation function and set
		$L=3, m_0=m_1=m_2=m_3=240$.}
	\label{fig:recovered}
\end{figure}

\begin{figure}
	\centering
	\includegraphics[width=0.8\linewidth]{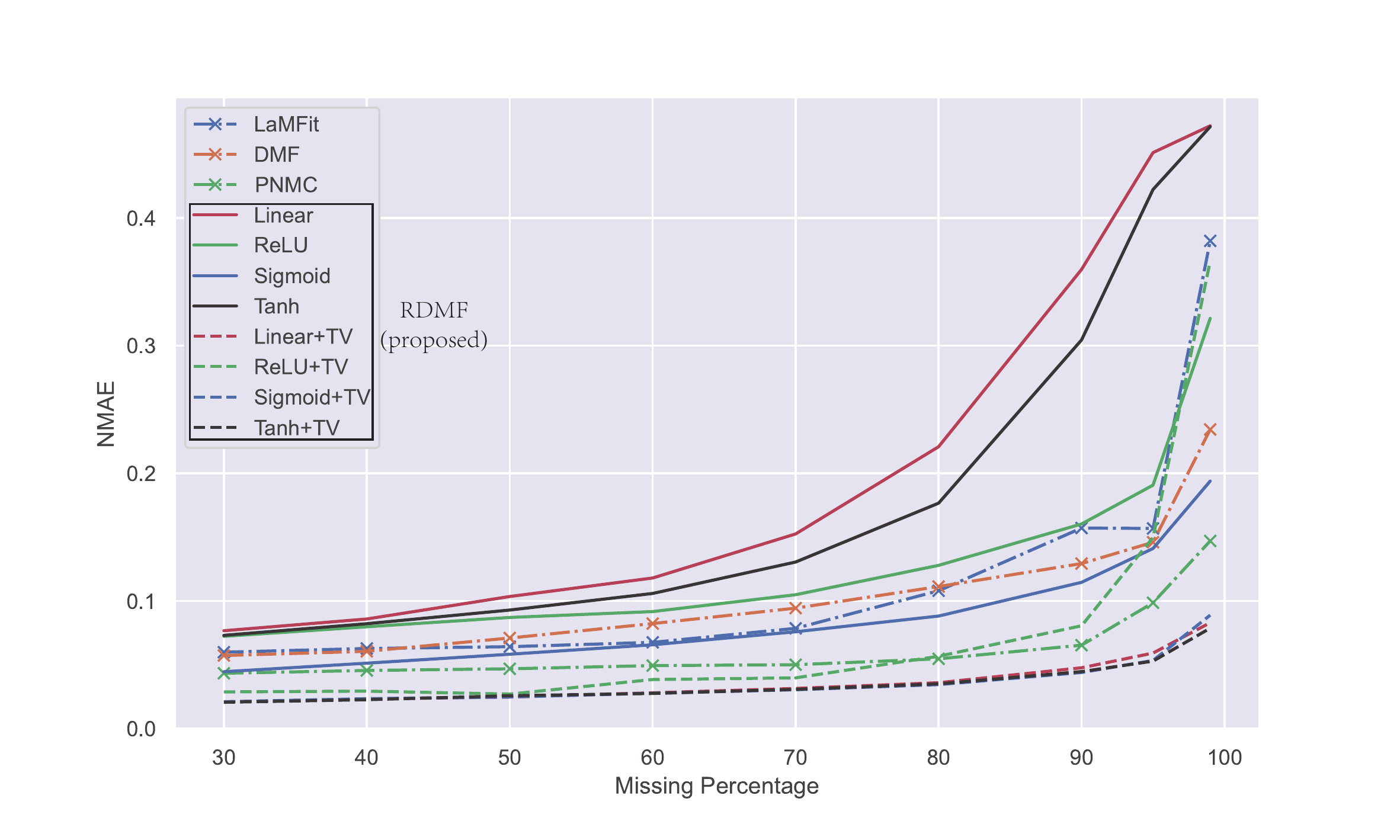}
	\caption{NMAE value of restored Cameraman with different models.
		Linear, ReLU, Sigmoid, and Tanh represent the RDMF model with different activation functions, and no TV term is added.
		Linear+TV, ReLU+TV, Sigmoid+TV and Tanh+TV represent RDMF
		model with different activation functions and TV term.}
	\label{fig:namecameraman}
\end{figure}

\begin{table}[htbp]
	\centering
	\caption{NMAE of different models in the restored textual image. Linear, ReLU, Sigmoid, and Tanh represent the RDMF model with different activation functions, and no TV term is added. Different rows of the table represent a different missing percentage of images.}
	\begin{tabular}{r|rrr|rrrr|rrrr}
		\toprule
		& \multicolumn{7}{c|}{Without TV}                       & \multicolumn{4}{c}{ With TV} \\
		\midrule
		& \multicolumn{1}{l}{LaMFit} & \multicolumn{1}{l}{DMF} & \multicolumn{1}{l|}{PNMC} & \multicolumn{8}{c}{Proposed} \\
		\cmidrule{1-1}\cmidrule{5-12}          & \multicolumn{1}{l}{\cite{Wen2012SolvingAL}} & \multicolumn{1}{l}{\cite{Fan2018MatrixCB}} & \multicolumn{1}{l|}{\cite{Yang2020ANP}} & \multicolumn{1}{l}{Linear} & \multicolumn{1}{l}{ReLU} & \multicolumn{1}{l}{Sigmoid} & \multicolumn{1}{l|}{Tanh} & \multicolumn{1}{l}{Linear} & \multicolumn{1}{l}{ReLU} & \multicolumn{1}{l}{Sigmoid} & \multicolumn{1}{l}{Tanh} \\
		\midrule
		30\%  & 0.0942 & 0.0593 & \multicolumn{1}{r}{0.0454} & 0.0459 & 0.044 & 0.0097 & 0.0249 & 0.0088 & 0.0153 & \textbf{0.0083} & 0.009 \\
		40\%  & 0.0969 & 0.0665 & \multicolumn{1}{r}{0.0531} & 0.0651 & 0.0526 & 0.0122 & 0.0327 & 0.0102 & 0.0243 & 0.0093 & \textbf{0.0091} \\
		50\%  & 0.0975 & 0.0778 & \multicolumn{1}{r}{0.0562} & 0.0851 & 0.0616 & 0.018 & 0.0437 & 0.012 & 0.0116 & \textbf{0.0104} & 0.0106 \\
		60\%  & 0.0998 & 0.0928 & \multicolumn{1}{r}{0.0608} & 0.1115 & 0.0859 & 0.0265 & 0.0593 & 0.0142 & 0.072 & \textbf{0.0126} & 0.013 \\
		70\%  & 0.1036 & 0.1112 & \multicolumn{1}{r}{0.0778} & 0.1485 & 0.118 & 0.0405 & 0.075 & 0.0179 & 0.0617 & \textbf{0.0159} & 0.0168 \\
		80\%  & 0.1148 & 0.1436 & \multicolumn{1}{r}{0.1068} & 0.2154 & 0.1485 & 0.065 & 0.1   & 0.0263 & 0.0871 & \textbf{0.0227} & 0.0237 \\
		90\%  & 0.1462 & 0.165 & \multicolumn{1}{r}{0.1198} & 0.438 & 0.1992 & 0.0991 & 0.1329 & 0.0474 & 0.1308 & \textbf{0.0392} & 0.0414 \\
		95\%  & 0.1605 & 0.1804 & \multicolumn{1}{r}{0.1334} & 0.7105 & 0.2329 & 0.1233 & 0.186 & 0.0751 & 0.156 & \textbf{0.0578} & 0.0645 \\
		99\%  & 0.3586 & 0.2663 & \multicolumn{1}{r}{0.1357} & 0.8729 & 0.4635 & 0.1925 & 0.856 & 0.1343 & 0.6995 & \textbf{0.1023} & 0.1178 \\
		\bottomrule
	\end{tabular}%
	\label{tab:nmae of text}%
\end{table}%

\subsection{Wider is better}
Different from LaMFit and DMF, which constrain small $m_1$ and $m_2$ to ensure the low-rank property of the restored images.
We default choose $m_0=m_1=m_2=m_3=240$ in the before mentioned experiments.
Arora et al. \cite{Arora2019ImplicitRI} claim that when using gradient descent to optimize the loss function (\ref{eq:implicit}).
The recovered matrix implicitly tends to be low-rank.
This implicit low-rank regularization makes our model work with the large $m_1$ and $m_2$ possible.
In general, a wider neural network has a stronger express ability.
It is reasonable to expect a larger $m_1$ and $m_2$ will lead to better performance.
The cameraman of missing $50\%$ pixels with RDMF is restored in this experiment.
Fig. \ref{fig:width} shows the NMAE with different $m_1$ and $m_2$.
A larger dimension indeed lead to a better performance on both numerically indicators (Fig. \ref{fig:width}(a)) and visually quality (Fig. \ref{fig:width}(b)). 
These results illustrate that the wider neural network can capture more information than a narrower one.

When $m_1=m_2=240$, the NMAE achieves least.
Continue increasing $m_1$ and $m_2$, RDMF performs same as the case of $m_1=m_2=240$.
In LaMFit and DMF, a proper low dimension is essential for the restored performance. This assumption means that different problems have a different optimal dimension.
In RDMF, we choose $m_0=m_1,m_2=m_3$, and need not worry about a large dimension that will lead to a bad performance.

\begin{figure}
	\centering
	\includegraphics[width=0.8\linewidth]{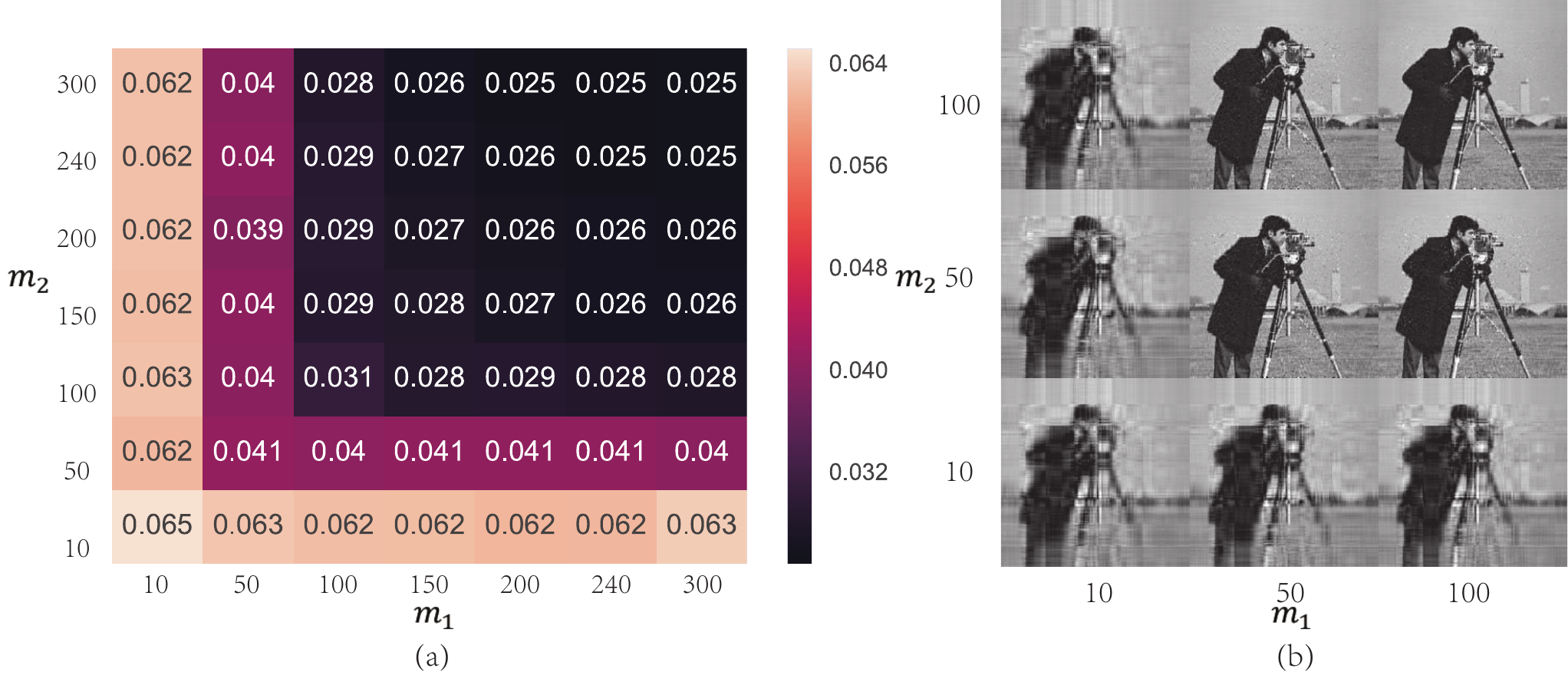}
	\caption{We set $L=3$ and choose the linear activation function.
		The input image of our model is that Cameraman missed $50\%$ pixels randomly. $m_1$ and $m_2$ are changed, the value of the heatmap
		in (a) represent the NMAE of the restored image.
		Part of the restored image with different $m_1,m_2$ are shown in (b).}
	\label{fig:width}
\end{figure}

\section{Discussion}
According to the numerical results,
TV regularization term plays a quite important role in image restoration problems.
Intuitively, TV regularization term forces the reconstructed image to have the smoothness property.
However, the influence of TV on the implicit low-rank regularization is still unknown.
In this section, we follow Arora et al. \cite{Arora2019ImplicitRI} to obtain an explicitly expression for the dynamics of singular values of the product matrix for our model with a specific regularization term.

To make the statements clear, we still use notations in Section \ref{sec..model} and introduce the following assumption.
\begin{assump}\label{assump..balance}
    Factor matrices are balanced at initialization, i.e.,
    $${\bm{W}^{[l+1]}}^{\top}(0) \bm{W}^{[l+1]}(0)=\bm{W}^{[l]}(0) {\bm{W}^{[l]}}^{\top}(0) ,\quad l=0, \ldots, L-2.$$
\end{assump}
Under the above assumption, Arora et al. studied the gradient flow of product matrix $\bm{W}(t)$ with the non-regularized risk function $R_{\bm{\Omega}}$,
i.e., 
\begin{equation}
    \dot{\bm{W}}^{[l]}(t)=-\frac{\partial}{\partial \bm{W}^{[l]}} R_{\bm{\Omega}}\left(\bm{W}(t)\right), \quad t \geq 0, \quad l=0, \ldots, L-1,\label{eq..DynamicsNonReg}
\end{equation}
where $\bm{W}=\bm{W}^{[L-1]}\bm{W}^{[L-2]}\ldots\bm{W}^{[1]}\bm{W}^{[0]}$. The empirical risk
$R_{\bm{\Omega}}(\cdot)$ can be any analytic function of entries of the matrices, not necessarily the Frobenious norm.
Thanks to the analyticity of $R_{\bm{\Omega}}(\cdot)$, we have the following singular value decomposition where each matrix is an analytic function of time.
\begin{lem}[\cite{Arora2019ImplicitRI}]
    The product matrix $\bm{W}(t)$ can be expressed as:
    \[
    \bm{W}(t)=\bm{U}(t) \bm{S}(t) \bm{V}^{\top}(t),
    \]
    where $\bm{U}(t) \in \mathbb{R}^{d_{\rm o}, \min \left\{d, d_{\rm o}\right\}}, 
    \bm{S}(t) \in \mathbb{R}^{\min \left\{d, d_{\rm o}\right\}, 
    	\min \left\{d, d_{\rm o}\right\}}$, and 
    $\bm{V}(t) \in \mathbb{R}^{d, \min \left\{d, d_{\rm o}\right\}}$ are analytic
    functions of t; and for every t, the matrices $\bm{U}(t)$ and 
    $\bm{V}(t)$ have orthonormal columns, while $\bm{S}(t)$ is 
    diagonal (its diagonal entries may be negative and may appear in any order).
\end{lem}

The diagonal entries of $\bm{S}(t),$ which we denote by
$\sigma_{1}(t), \ldots, \sigma_{\min \left\{d, d_{\rm o}\right\}}(t),$ 
are signed singular values of $\bm{W}(t)$. The columns of $\bm{U}(t)$ and $\bm{V}(t)$,
denoted by $\bm{u}_{1}(t), \ldots, \bm{u}_{\min \left\{d, d_{\rm o}\right\}}(t)$ 
and $\bm{v}_{1}(t), \ldots, \bm{v}_{\min \left\{d, d_{\rm o}\right\}}(t),$ 
are the corresponding left and right singular vectors respectively.

\begin{prop}[{\cite[Theorem 3]{Arora2019ImplicitRI}}]\label{prop..SinuglarValue}
    Consider the dynamics \eqref{eq..DynamicsNonReg} with initial data satisfying Assumption \ref{assump..balance}. Then the signed singular values $\sigma_r(t)$ of the product matrix $\bm{W}(t)$ evolve by:
    \begin{equation}
    \label{eq:arora}
    \dot{\sigma}_r(t)=-L \left(\sigma_r^{2}(t)\right)^{1-\frac{1}{L}} 
    \left\langle\nabla_{\bm{W}} R_{\bm{\Omega}}(\bm{W}(t)), \bm{u}_r(t) 
    \bm{v}_r^{\top}(t)\right\rangle, \quad r=1, \ldots, 
    \min \left\{d, d_{\rm o}\right\}.
    \end{equation}
\end{prop}
If the matrix factorization is non-degenerate,
i.e., has depth $L \geq 2,$ the singular values 
need not be signed (we may assume $\sigma_r(t) \geq 0$ for all $t$ ).

Arora et al. claimed the terms $\left(\sigma_r^2(t)\right)^{1-\frac{1}{L}}$ enhance the movement of large singular values, and on the other hand, attenuate that of small ones.
The enhancement/attenuation becomes more significant as $L$ grows.

This explanation of the implicit low-rank regularization of gradient descent is far from complete because the terms $\left\langle\nabla_{\bm{W}} R_{\bm{\Omega}}(\bm{W}(t)), \bm{u}_r(t)\bm{v}_r^{\top}(t)\right\rangle$ are unkonwn. Hence
Arora et al. further discussed a particular case to illustrate the low-rank regularization further.
In this paper, we give an illustration of the implicit low-rank regularization of TV-like regularization term.
To simplify the proof, we modify from
$
    R_{\mathrm{TV}}(\bm{W})=\left\|\bm{W}\right\|_\mathrm{TV}=\sum_{i,j}\sqrt{(D_x \bm{W}(i,j))^2
    	+(D_y \bm{W}(i,j))^2}
$
to 
$
    \tilde{R}_\mathrm{TV}(\bm{W})=\sum_{i,j}(D_x \bm{W}(i,j))^2 +(D_y \bm{W}(i,j))^2
$. The new TV-like regularization term has a negligible influence in experiments.
Without loss of generality, we assume $d=d_{\rm o}$, and then we have the following proposition.
\begin{prop}\label{prop..DynamicsReg}
    Consider the following dynamics with initial data satisfying Assumption \ref{assump..balance}:
    \begin{align*}
        \dot{\bm{W}}^{[l]}(t)
        &= -\frac{\partial}{\partial \bm{W}^{[l]}} R\left(\bm{W}(t)\right), \quad t \geq 0, \quad l=0, \ldots, L-1,\label{eq..DynamicsReg}
    \end{align*}
    where $ R(\bm{W})= R_{\bm{\Omega}}(\bm{W})+\lambda \tilde{R}_\mathrm{TV}(\bm{W})$ with $\tilde{R}_\mathrm{TV}(\bm{W})=\left\|\bm{AW}\right\|_\mathrm{F}^2+\left\|\bm{WA}^{\top}\right\|_\mathrm{F}^2$, and
    	$
    	\bm{A}=
    	\left[
    	\begin{array}{cccc}
    	1&-1&&\\
    	&\ddots&\ddots&\\
    	&&1&-1\\
    	-1&&&1\\
    	\end{array}
    	\right]_{d\times d}
    	$. Then we have 
    \begin{equation}
        \bm{u}_r^{\top}\left(\nabla_{\bm{W}}\tilde{R}_\mathrm{TV}\right)\bm{v}_r = 2\sigma_r(\bm{u}_r^{\top}\bm{A}^{\top}\bm{Au}_r+\bm{v}_r^{\top}\bm{A}^{\top}\bm{Av}_r)
        =2\sigma_r\gamma_r(t),
    \end{equation}
    where 
    	$\bm{W}=\sum\limits_s\sigma_s \bm{u}_s \bm{v}_s^{\top},\gamma_r(t)=\left\|\bm{Au}_r\right\|_2^2+\left\|\bm{Av}_r\right\|_2^2$.
\end{prop}

\begin{proof}
    This is proved by direct calculation:
	$$
	\begin{aligned}
	\nabla_{\bm{W}}\tilde{R}_\mathrm{TV}
	&=\frac{\partial (\left\|\bm{AW}\right\|_\mathrm{F}^2+\left\|\bm{WA}^{\top}\right\|_\mathrm{F}^2)}{\partial \bm{W}}\\
	&=\frac{\partial \mathrm{tr}(\bm{AWW}^{\top}\bm{A}^{\top}+\bm{AW}^{\top}\bm{WA}^{\top})}{\partial \bm{W}}\\
	&=2\bm{WA}^{\top}\bm{A}+2\bm{A}^{\top}\bm{AW}\\
	&=2\sum\limits_s\sigma_s \bm{u}_s \bm{v}_s^{\top} \bm{A}^{\top}\bm{A}
	+2\bm{A}^{\top}\bm{A}\sum\limits_s\sigma_s \bm{u}_s \bm{v}_s^{\top}.
	\end{aligned}
	$$
	Note that
	$$
	\langle\bm{v}_s,\bm{v}_{s'}\rangle=\langle\bm{u}_s,\bm{u}_{s'}\rangle=\delta_{ss'}=
	\left\{
	\begin{array}{cc}
	1, & s=s',\\
	0. & s\neq s'.\\
	\end{array}\right.$$
	Therefore
	$$
	\begin{aligned}
	\bm{u}_r^{\top}\left(\nabla_{\bm{W}}\tilde{R}_\mathrm{TV}\right)\bm{v}_r 
	&= 2\sigma_r(\bm{u}_r^{\top}\bm{A}^{\top}\bm{Au}_r+\bm{v}_r^{\top}\bm{A}^{\top}\bm{Av}_r)\\
	&=2\sigma_r(\left\|\bm{A}\bm{u}_r\right\|_2^2+\left\|\bm{A}\bm{v}_r\right\|_2^2)\\
	&=2\sigma_r\gamma_r(t),
	\end{aligned}
	$$
	 where the term $\gamma_r(t)=\left\|\bm{Au}_r\right\|_2^2+\left\|\bm{Av}_r\right\|_2^2\geq 0$.
\end{proof}

\begin{cor}
    In the setting of Proposition \ref{prop..DynamicsReg}, we further have
    \begin{equation*}
        \dot{\sigma}_r(t)=-L \left(
        \sigma_r^{2}(t)\right)^{1-\frac{1}{L}} 
         \left\langle\nabla_{\bm{W}} R_{\bm{\Omega}}(\bm{W}(t)), \bm{u}_r(t)\bm{v}_r^{\top}(t)\right\rangle
        -2L\lambda \left(\sigma_r^2(t)\right)^{\frac{3}{2}-\frac{1}{L}} \gamma_r(t).
    \end{equation*}
\end{cor}
\begin{proof}
     Note that $R=R_{\bm{\Omega}}+\lambda 
     \tilde{R}_\mathrm{TV}$. The statement follows directly from Propositions \ref{prop..SinuglarValue} and \ref{prop..DynamicsReg}.
\end{proof}
The terms $\left(\sigma_r^2(t)\right)^{\frac{3}{2}-\frac{1}{L}}$ have a larger exponent than $\left(\sigma_r^2(t)\right)^{1-\frac{1}{L}}$.
The terms $\left(\sigma_r^2(t)\right)^{\frac{3}{2}-\frac{1}{L}}$ can enhance/attenuate singular values more significantly than $\left(\sigma_r^2(t)\right)^{1-\frac{1}{L}}$.
Specially, if we choose a large enough $\lambda$,
$R=R_{\bm{\Omega}}+\lambda \tilde{R}_\mathrm{TV}$ is dominated by $\tilde{R}_\mathrm{TV}$.
The term $\tilde{R}_\mathrm{TV}=0$ if and only if all the entries of $\bm{W}$ have the same value, and the rank of $\bm{W}$ equals $1$ or $0$.
This is consistent with the above analysis.
Our theoretic discussion is not complete because both the terms $\gamma_r(t)$ and $\left\langle\nabla_{\bm{W}} R_{\bm{\Omega}}(\bm{W}(t)), \bm{u}_r(t)\bm{v}_r^{\top}(t)\right\rangle$ are unknown and the regularization term is not the exact TV norm.
We design some experiments to show our conclusion about the low-rank regularization property of (the exact) TV norm in general.

We calculate the effective rank \cite{Roy2007TheER} of RDMF with different activation functions.
Both of the RDMF with and without TV regularization terms are calculated.
Fig. \ref{fig:rankcameraman} shows the effective rank of restored Cameraman at different missing percentage and various RDMF.
Tab. \ref{tab:rank of text} shows the effective rank of the restored textual image at different missing percentage by different RDMF.
Almost all the models with TV have a lower effective rank than these models without TV.
This phenomenon is consistent with Corollary 1:
TV norm has a stronger implicit low-rank regularization during training.
Another important phenomenon is that almost all the models which achieve good performances (Fig. \ref{fig:namecameraman}, Tab. \ref{tab:nmae of text}) have a low effective rank.
This phenomenon indicates that the implicit low-rank regularization indeed produces better-restored quality.
Last but not least, the model with a nonlinear activation function also keeps the restored matrix's low-rank property.
This phenomenon still does not have theoretical analysis, and it needs further research in the future.

Another conclusion of Corollary 1 is that a lager $L$ will enhance the low-rank property.
Fig. \ref{fig:depth} shows that when $L=2$, the model performance badly, especially when the missing percentage is significant.
When $L=3,4$, the rank of the restored image is lower than the case of $L=2$.
Furthermore, the restored image of the massive $L$ is closer to the real image than the small one.

\begin{figure}
	\centering
	\includegraphics[width=0.8\linewidth]{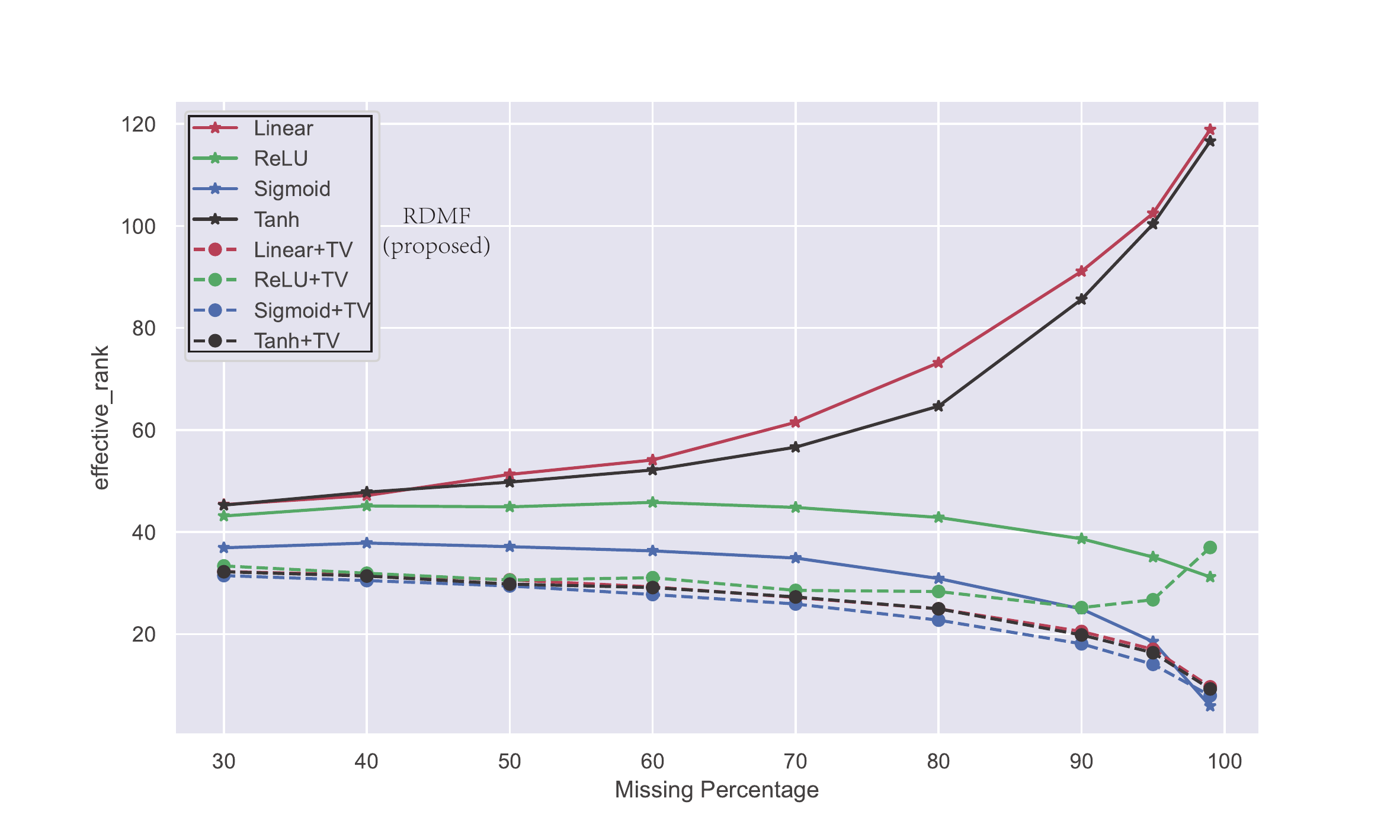}
	\caption{Effective rank \cite{Roy2007TheER} of different models in recovering Cameraman.
		Linear, ReLU, Sigmoid and Tanh represent RDMF model with
		different activation functions, and no TV term is added.
		Linear+TV, ReLU+TV, Sigmoid+TV and Tanh+TV represent RDMF
		model with different activation functions and TV term.}
	\label{fig:rankcameraman}
\end{figure}
\begin{table}[htbp]
	\centering
	\caption{The effective rank of the restored textual image with different models. 
		Linear, ReLU, Sigmoid, and Tanh represent the RDMF model with different activation functions, and no TV term is added.
		Different rows of the table represent the different missing percentage of images.}
	\begin{tabular}{r|rrrr|rrrr}
		\toprule
		& \multicolumn{4}{c|}{Without TV} & \multicolumn{4}{c}{With TV} \\
		\cmidrule{2-9}          & \multicolumn{1}{l}{Linear} & \multicolumn{1}{l}{ReLU} & \multicolumn{1}{l}{Sigmoid} & \multicolumn{1}{l|}{Tanh} & \multicolumn{1}{l}{Linear} & \multicolumn{1}{l}{ReLU} & \multicolumn{1}{l}{Sigmoid} & \multicolumn{1}{l}{Tanh} \\
		\midrule
		30\%  & 19.5998 & 17.2613 & 12.992 & 18.0705 & 12.5006 & 13.0034 & \textbf{12.0801} & 12.7237 \\
		40\%  & 23.4589 & 19.1827 & 13.7183 & 20.4861 & 12.8015 & 13.7993 & \textbf{12.2395} & 12.6712 \\
		50\%  & 26.7782 & 21.2089 & 14.5887 & 23.4846 & 13.0383 & 12.8354 & \textbf{12.3894} & 12.8573 \\
		60\%  & 30.2439 & 23.1571 & 15.9678 & 27.6808 & 13.0675 & 19.4384 & \textbf{12.4225} & 13.062 \\
		70\%  & 34.8062 & 23.4533 & 17.6626 & 29.816 & 13.105 & 17.9026 & \textbf{12.5056} & 13.108 \\
		80\%  & 41.7717 & 23.1704 & 19.0667 & 30.1278 & 13.1146 & 19.1531 & \textbf{12.3723} & 12.9568 \\
		90\%  & 64.4488 & 20.4023 & 15.8515 & 28.7328 & 12.5982 & 18.9162 & \textbf{11.3012} & 11.9169 \\
		95\%  & 84.3133 & 17.6296 & 11.6738 & 35.9907 & 11.8874 & 15.52 & \textbf{9.9576} & 10.176 \\
		99\%  & 105.9509 & 14.3999 & \textbf{3.695} & 111.0835 & 7.1012 & 39.2951 & 5.4971 & 6.6454 \\
		\bottomrule
	\end{tabular}%
	\label{tab:rank of text}%
\end{table}%

\begin{figure}
	\centering
	\includegraphics[width=0.8\linewidth]{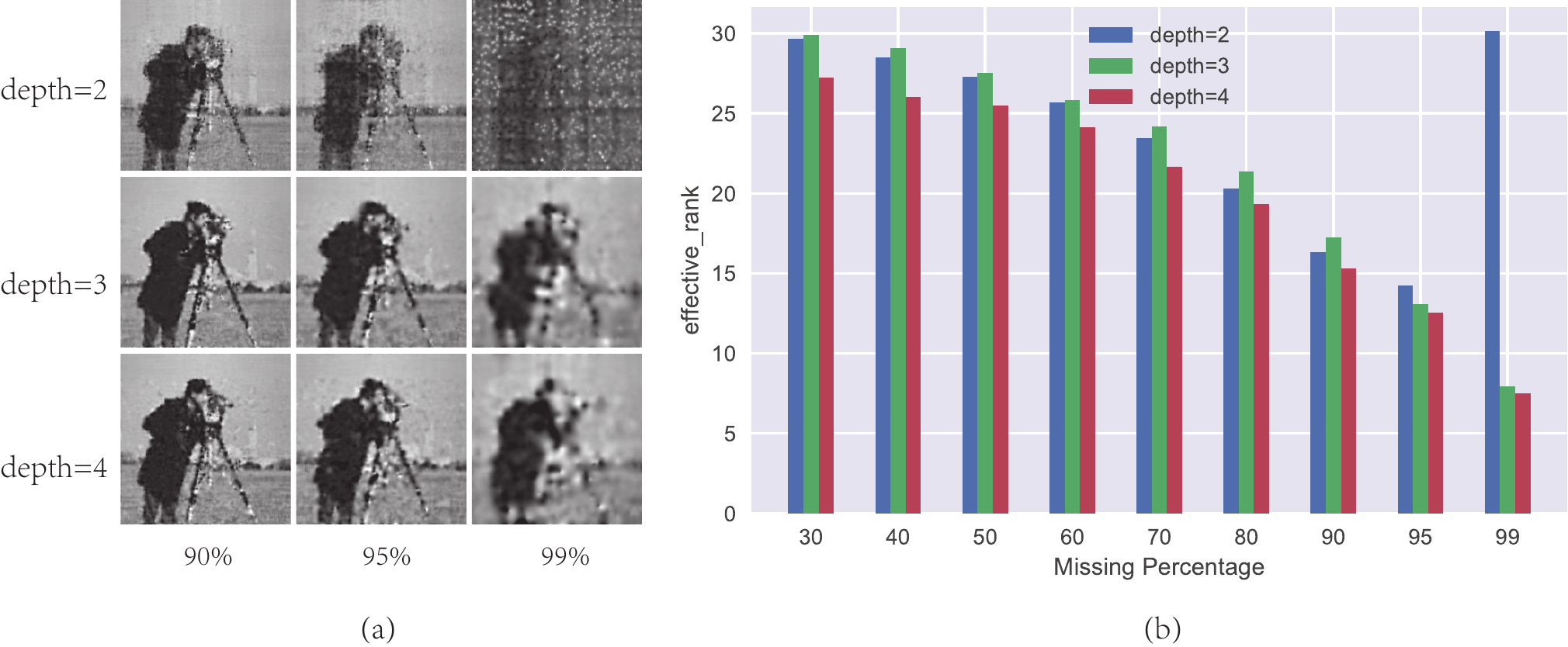}
	\caption{In this experiment, we restored the cameraman with a different percentage randomly missed pixels. 
		We set $m_i=240$ and change $L$ from 2 to 4.
		Part of the restored results is shown in (a).
		All the effective rank results are shown in (b).}
	\label{fig:depth}
\end{figure}

\section{Conclusion}
In this paper, we present an RDMF model for matrix completion.
This model shows that linear factor models with three equal-size matrices have
enough modeling capacity. 
The critical point of the proposed model is that we combine the implicit bias in deep networks with TV.
We conduct experiments on image tasks to demonstrate that RDMF outperforms state-of-the-art matrix completion methods.
We also conduct experiments to test the impact of principal components in
the proposed model. 
The results show that the TV regularization term can significantly improve restored performance.
The improvement of adding nonlinear activation functions in a model can be ignored compared with the TV regularization term.
A wider and deeper model preserves more information than a narrow and shallow one.
Finally, we have discussed the underlying mechanism of RDMF based on a deep linear model.
The TV regularization term has a stronger implicit low-rank regularization property.
The designed experiments confirm our discussion well.

In the future, we plan to extend the new model to other matrix completion problems.
The regularization term of our proposed model is related to the matrix completion
problem firmly. We plan to explore the reinforcement learning framework to auto choose the regularization term based on the completion problem.


\bibliographystyle{unsrt}  
\bibliography{references}  


\end{document}